\documentclass{article}
\usepackage{graphicx} 

\usepackage[utf8]{inputenc}
\usepackage{amsmath,amsfonts,amssymb}
\usepackage{mathtools}
\usepackage{stackrel}

\usepackage{tikz}
\usetikzlibrary{positioning}

\usepackage[doi=false,isbn=false,maxbibnames=99,url=false,style=alphabetic]{biblatex}

\usepackage{url}
\usepackage{amsthm}
\usepackage{mathrsfs}
\usepackage{bbm}
\usepackage{dsfont}
\usepackage{color}
\usepackage{comment}
\usepackage[top=1in,bottom=1in,left=1in,right=1in]{geometry}

\newtheorem{theorem}{Theorem}[section]
\newtheorem{conjecture}[theorem]{Conjecture}
\newtheorem{corollary}[theorem]{Corollary}
\newtheorem{proposition}[theorem]{Proposition}
\newtheorem{lemma}[theorem]{Lemma}

\usepackage[normalem]{ulem}
\usepackage{hyperref}
\usepackage{cleveref}[2012/02/15]
\usepackage{caption}

\newcommand{\bbE}{\mathbb{E}}
\newcommand{\bbP}{\mathbb{P}}

\newcommand{\pa}[1]{\left(#1\right)}
\newcommand{\ac}[1]{\left\{#1\right\}}
\newcommand{\cro}[1]{\left[#1\right]}

\title{Phase Transition for Stochastic Block Model with more than $\sqrt{n}$ Communities (II)}

\author{Alexandra Carpentier\footnote{Institut für Mathematik -- Universität Potsdam, Potsdam, Germany. Alexandra.Carpentier@uni-potsdam.de},   \  Christophe Giraud\footnote{Laboratoire de Math\'ematiques d'Orsay, Universit\'e Paris-Saclay, CNRS, France. Christophe.Giraud@universite-paris-saclay.fr} \ 
and Nicolas Verzelen\footnote{INRAE, Institut Agro, MISTEA, Univ. Montpellier, France. Nicolas.Verzelen@inrae.fr} }
\date{}

\begin{document}

\maketitle

\begin{abstract}
A fundamental theoretical question in network analysis is to determine under which conditions community recovery is possible in polynomial time in the Stochastic Block Model (SBM). 
When the number $K$ of communities remains smaller than $\sqrt{n}$ ---where $n$ denotes the number of nodes in the observed graph---, non-trivial community recovery is possible in polynomial time above, and only above, the Kesten--Stigum (KS) threshold, originally postulated using arguments from statistical physics.

When $K \geq \sqrt{n}$, Chin, Mossel, Sohn, and Wein~\cite{pmlr-v291-chin25a} recently proved that, in the \emph{sparse regime}, community recovery in polynomial time is achievable below the KS threshold by counting non-backtracking paths. 
This finding led them to postulate a new threshold for the many-communities regime $K \geq \sqrt{n}$. 
Subsequently, Carpentier, Giraud, and Verzelen~\cite{carpentier2025phase} established the failure of low-degree polynomials below this new threshold across \emph{all density regimes}, and demonstrated successful recovery above the threshold in \emph{certain moderately sparse} settings. 
While these results provide strong evidence that, in the many community setting, the computational barrier for community recovery  lies at the threshold proposed in~\cite{pmlr-v291-chin25a}, the question of achieving recovery above this threshold still remains open in most density regimes.

The present work is a follow-up to~\cite{carpentier2025phase}, in which we prove Conjecture~1.4 stated therein by: \\
1- Constructing a family of motifs satisfying specific structural properties; and\\
2- Proving that community recovery is possible above the proposed threshold by counting such motifs.\\ 
Together with~\cite{pmlr-v291-chin25a} and~\cite{carpentier2025phase}, 
our results provide a complete picture of the computational barrier for community recovery in the SBM with $K \geq \sqrt{n}$ communities. 
They also indicate that, in moderately sparse regimes, the optimal algorithms appear to be fundamentally different from spectral methods.
\end{abstract}

\section{Introduction}

Network analysis aims to understand the structure of random interactions among individuals or objects. A network over $n$ individuals is typically represented as an undirected graph with $n$ nodes, where edges encode observed binary interactions between pairs of individuals. Examples include social networks (friendships or follower relations), biological networks (gene--gene or protein--protein interactions), and information networks (such as email exchanges or citation graphs).
A fundamental statistical model for such data is the \emph{Stochastic Block Model} (SBM), introduced by Holland, Laskey, and Leinhardt~\cite{holland1983stochastic}. In the SBM, the $n$ nodes are partitioned into $K<n$ latent groups or \emph{communities}. Conditional on these (unobserved) community labels, edges are generated independently, with probabilities depending only on the communities of the connected nodes. In its simplest form, the within-community connection probability is $p$, while the between-community connection probability is $q$, with $p>q$. This model captures community structure in a mathematically tractable way and has been widely studied in statistics, computer science, and probability theory~\cite{AbbeReview2018}.

The central inferential task in SBM analysis is to recover the latent community memberships from an observed network. In the sparse regime, where $q$ scales as $1/n$, the seminal paper of Decelle, Krzakala, Moore, and Zdeborov\'a~\cite{Decelle2011} conjectured ---based on the replica heuristic from statistical physics--- 
that, when $K$ is fixed and $n\to+\infty$, community recovery in polynomial-time is possible only above the Kesten--Stigum (KS) threshold
\begin{equation}\label{eq:KS_intro}
\frac{n\lambda^2}{K\lambda + K^2q} > 1.
\end{equation}
Considerable effort has been devoted to confirming this conjecture. Polynomial-time algorithms provably achieving non-trivial recovery above the KS threshold have been developed in~\cite{massoulie2014community,mossel2018proof,bordenave2015non,AbbeSandon2015a,pmlr-v291-chin25a}.
Some hardness results below the KS threshold have been established within the Low-Degree Polynomial  framework~\cite{luo2023computational,SohnWein25,ding2025low}, providing strong evidence for computational intractability in this regime.
While most prior research has focused on small numbers of communities, recent attention has turned to the regime where the number of communities $K$ grows with $n$. Early work by Chen and Xu~\cite{chen2014statistical} proposed polynomial-time algorithms achieving partial recovery near the KS threshold. Subsequent studies~\cite{stephan2024community,pmlr-v291-chin25a} have shown that when $K=o(\sqrt{n})$, recovery above the KS threshold is possible in polynomial time, 
while low-degree polynomials fail below the KS threshold, thus confirming the predictions of~\cite{Decelle2011} in this regime.

When $K \gg \sqrt{n}$, the computational barrier for recovery  remained poorly understood until recently. 
Chin, Mossel, Sohn, and Wein~\cite{pmlr-v291-chin25a} made a breakthrough by showing that, in the sparse setting (where $q$ scales as $1/n$)  with $K \gg \sqrt{n}$ communities, partial recovery is achievable \emph{below} the KS threshold using a polynomial-time algorithm based on non-backtracking statistics. 
Their algorithm succeeds in the sparse regime as soon as
\begin{equation}\label{eq:new_intro}
\lambda \gtrsim_{\log} \left(q + \frac{\lambda}{K}\right)^{1 - \log_n(K)},
\end{equation}
where $\gtrsim_{\log}$ hides polylogarithmic factors, and $\log_n$ denotes the logarithm in base $n$, so that $K = n^{\log_n(K)}$. 
The failure of low-degree polynomials below the threshold~\eqref{eq:new_intro} was later established by Carpentier, Giraud, and Verzelen~\cite{carpentier2025phase}, not only in the sparse setting ($q$ scaling as $1/n$) considered in~\cite{pmlr-v291-chin25a}, but across all density regimes satisfying $q = o(1)$. 
Furthermore,~\cite{carpentier2025phase} shows that community recovery is possible above the threshold~\eqref{eq:new_intro} using $m$-clique counting when $q \asymp n^{-2/(m+1)}$, for $m \in \{3,4,5,\ldots\}$, and by counting self-avoiding paths of length $m-1$ when $q \asymp n^{-(m-2)/(m-1)}$, for $m \in \{3,4,5,\ldots\}$. 
For the remaining density regimes, Conjecture~1.4 of~\cite{carpentier2025phase} predicts that community recovery remains possible above the threshold~\eqref{eq:new_intro} through the enumeration of specific motifs satisfying certain structural properties.

The present paper is a follow-up to~\cite{carpentier2025phase}, dedicated to proving this conjecture. 
Our main contributions are :
\begin{enumerate}
\item To construct a new family of motifs fulfilling the structural properties required in Conjecture~1.4 of~\cite{carpentier2025phase};
\item To prove that an algorithm based on counting such motifs successfully recovers the communities above the threshold~\eqref{eq:new_intro}.
\end{enumerate}
These results complement the contributions of~\cite{pmlr-v291-chin25a} and~\cite{carpentier2025phase}, 
and together establish that the computational barrier for community recovery in the SBM with $K \geq \sqrt{n}$ communities lies at~\eqref{eq:new_intro}. 
Moreover, they indicate that the optimal algorithms in this regime appear to be unrelated to spectral methods, echoing similar observations made for Gaussian mixture models~\cite{Even24}.

\paragraph{Organisation and notation.}
In Section~\ref{sec:setting}, we describe  the statistical setting, and we recall the Conjecture~1.4 of~\cite{carpentier2025phase} that we prove in this paper. In Section~\ref{sec:main}, we introduce the family of motifs used for successful recovery of the communities above the threshold~\eqref{eq:new_intro}, and we state our main results. The proofs are provided in Section~\ref{sec:key:blow-up} and~\ref{sec:mean:variance:blow_up}.

We denote by $\log(\cdot)$ the natural logarithm, and by $\log_{n}(\cdot)$ the logarithm in base $n$, i.e.\ $K=n^{\log_{n}(K)}$.
For better readability, we write ${\bf 1}\ac{A}$ for the indicator function ${\bf 1}_{A}$ of the set $A$.
We also use the short notation $[D]$ for the set $\ac{1,\ldots,D}$. 
In the discussion of the results, we use the symbol $a\lesssim_{D} b$ (respectively $a\lesssim_{\log} b$) to state that $a$ is smaller than $b$ up to a possible polynomial factor in $D$ (resp. up to a poly-log factor), and we use $a\asymp b$ to state that $a\lesssim b \lesssim a$.

\section{Setting and goal}\label{sec:setting}

\subsection{Stochastic block model and motif counting}

\paragraph{Stochastic Block Model.} Let $K$ be an integer in $[2,n]$, and $p,q\in (0,1)$, with $p>q$. We assume that $Y^*$ is sampled according to the following distribution:
\begin{enumerate}
\item  $z_{1},\ldots,z_{n}\stackrel{i.i.d.}{\sim} \text{Uniform}\ac{1,\ldots,K}$;
\item Conditionally on $z_{1},\ldots,z_{n}$, the entries $(Y^*_{ij})_{i<j}$ are sampled independently, with $Y^*_{ij}$ distributed as a Bernoulli random variable with parameter $q+\lambda{\bf 1}\ac{z_{i}=z_{j}}$, where $\lambda=p-q>0$. 
\end{enumerate}
To get simpler formulas, in the remaining of the paper, we will work with the quasi-centered adjacency matrix
\begin{equation}\label{eq:definition:Y}
Y_{ij} = (Y^{*})_{ij}- q \ , \text{ for any }1\leq i < j \leq n\enspace .
\end{equation}
We will also use the notation
\begin{equation}\label{eq:bar-proba}
\bar q=q(1-q),\quad \text{and}\quad \bar p= \bar q +\lambda (1-2q).
\end{equation}

\paragraph{Community recovery.}
Our goal is to recover the partition $\mathcal{C}^*=\ac{\mathcal{C}^*_{1},\ldots,\mathcal{C}^*_{K}}$  of $[n]$ induced by the community assignment $z_{1},\ldots,z_{n}$, i.e.
$$\mathcal{C}^*_{k}:=\ac{i\in [n]:z_{i}=k},\quad k=1,\ldots,K.$$
We define the membership matrix 
$$[x_{ij}]_{1\leq i<j \leq n}=[\mathbf{1}\ac{i,j\ \text{are in the same community}}-1/K]_{1\leq i<j \leq n}.$$
If, for each $1\leq i<j \leq n$, we produce an estimator $\hat x_{ij}$ equal to $x_{ij}$ with probability $1-o(n^{-2})$, then the connected components of $\hat x$ coincide with $\mathcal{C}^*$ with probability $1-o(1)$. Hence, we focus henceforth on the problem of estimating $x_{ij}$. Our algorithm is based on motif counting, introduced below.

\paragraph{Motif counting.} Let $G=(V,E)$ be a graph, henceforth referred as motif, on a node set $V=\ac{v_{1},\ldots,v_{D}}$. 
Fix $1\leq i < j\leq n$ and denote $\Pi_{ij}$ the  set of injections $\pi:V\to \ac{1,\ldots,n}$ fulfilling $\pi(v_{1})=i$ and $\pi(v_{2})= j$. Then, we define the motif counting  
$$P_G = \sum_{\pi\in \Pi_{ij}} P_{G,\pi},\quad \text{with} \quad P_{G,\pi}(Y)= \prod_{(v,v')\in E} Y_{\pi(v),\pi(v')}\enspace .$$
We refer to $P_{G}$ as a motif counting since, when applied to the vanilla adjacency matrix $Y^*$, the polynomial $P_{G}(Y^*)$ merely counts the occurence in $Y^*$ of the motif $G$ ``attached''
at $i$ and $j$.

\subsection{Conjecture 1.4 from~\cite{carpentier2025phase}}

Our goal is to prove Conjecture~1.4 of~\cite{carpentier2025phase} that  predicts that community recovery is possible above the threshold~\eqref{eq:new_intro} through the enumeration of specific motifs satisfying certain structural properties. Before recalling the details of the conjecture, we provide some heuristic arguments and sketchy computations  motivating it.

To start with, we notice that the signal Condition~\eqref{eq:new_intro} can be written 
$\lambda\geq_{\log} \pa{\lambda\over K}^{1-\log_{n}(K)}\vee q^{1-\log_{n}(K)}$ or equivalently 
\begin{equation}\label{eq:new_intro2}
\lambda\geq_{\log} \pa{K\over n}\vee q^{1-\log_{n}(K)}.
\end{equation}
In the density regime $q\asymp n^{-1/r}$ with $r> 1$, under the condition $\lambda\geq_{\log} q^{1-\log_{n}(K)}$, we have
\begin{equation}\label{eq:cond:lambda/q}
\lambda^r\geq_{\log} q^r K\asymp {K\over n}.
\end{equation}
Combining \eqref{eq:cond:lambda/q} and \eqref{eq:new_intro2}, we then observe that when $q\asymp n^{-1/r}$ with $r> 1$ the signal Condition~\eqref{eq:new_intro} reduces to
\begin{equation}\label{eq:new_intro3}
\lambda\geq_{\log} q^{1-\log_{n}(K)}.
\end{equation}

Let us motivate the structural properties of $G$ appearing in the conjecture. 
First of all, we can readily check that for non-connected motifs, the count $P_{G}(Y)$ has zero correlation with $x_{ij}$.
Hence, we restrict our attention to connected motifs $G$. 
Introducing the notation
\begin{equation}\label{eq:conditional:P}
\bbP_{ij}:=\bbP\cro{\cdot|z_{i}=z_{j}},\quad\text{and}\quad \bbP_{\not ij}:=\bbP\cro{\cdot|z_{i}\neq z_{j}},
\end{equation}
a direct computation gives $\bbE_{\not ij}\cro{P_{G}}=0$. So,
for testing  $z_{i}=z_{j}$ against $z_{i}\neq z_{j}$ with a small error, we seek for a connected motif $G$ fulfilling
\begin{equation}\label{eq:test:success}
\bbE_{ij}\cro{P_{G}}^2 \gg  \mathrm{var}_{ij}(P_{G}) \vee  \mathrm{var}_{\not ij}(P_{G}).
\end{equation}
The expectation can be simply evaluated (see Equations~\eqref{eq:proof:ingredient:basis:mean} and~\eqref{eq:mean:12:pi:blow-up})
\begin{equation}\label{eq:E12:intro}
\bbE_{ij}\cro{P_{G}} \asymp \pa{n \lambda^{|E|/(|V|-2)}\over K}^{|V|-2}=\pa{n \lambda^{r}\over K}^{|V|-2},
\end{equation}
with $r:=|E|/(|V|-2)$.
So, according to~\eqref{eq:cond:lambda/q}, we have $\bbE_{ij}\cro{P_{G}}\gtrsim 1$. Given a labeling $\pi$ of $V$, we write $\pi(G)$ for the corresponding labelled graph with nodes $\pi(V)$ and edges $\pi(E)=\ac{(\pi(u),\pi(v)):(u,v)\in E}$. We define $|E^{\neq}|$ as the number of edges in $\pi(G)$  between nodes of distinct communities.
A direct computation gives
\begin{align}
\bbE_{ij}\cro{P_{G,\pi}^2} & =\bbE_{ij}\cro{\prod_{\substack{(i,j)\in \pi(E) \\ z_{i} =z_{j}}}\bar p\prod_{\substack{(i,j)\in \pi(E) \\ z_{i} \neq z_{j}}}\bar q}
\ =\  \bbE_{ij}\cro{ \bar p^{|E|-|E^{\neq}|}\bar q^{|E^{\neq}|}} \ \asymp\ \lambda^{|E|}\ \bbE_{ij}\cro{\pa{q\over \lambda}^{|E^{\neq}|}},\label{eq:moment2}
\end{align}
where $\bar p$, $\bar q$ are defined in~\eqref{eq:bar-proba}, with $\bar p\sim \lambda$ since $\lambda \geq_{\log} K^{1/r} q$ according to~\eqref{eq:cond:lambda/q}.
As explained in~\cite{carpentier2025phase}, we expect that the leading terms in the variance decomposition 
$$\mathrm{var}_{ij}(P_{G})=\sum_{\pi,\pi'} \mathrm{cov}_{ij}(P_{\pi,G},P_{\pi',G}),$$
are those for which $\pi(G)=\pi'(G)$.
 Writing $|\mathrm{Aut}(G)|\leq |V|^{|V|}$ for the number of automorphisms of $G$, and recalling \eqref{eq:moment2} and $|E|=r(|V|-2)$, we then expect to have
\begin{align}
\mathrm{var}_{ij}(P_{G})\vee  \mathrm{var}_{\not ij}(P_{G})&\stackrel{?}{\lesssim} \sum_{\pi\in \Pi_{V}} |\mathrm{Aut}(G)|\,\mathrm{var}_{ij}(P_{G,\pi}) 
\ \lesssim\ \pa{n\lambda^r}^{|V|-2}\ \bbE_{ij}\cro{(q/\lambda)^{|E^{\neq}|}},\label{eq:basic:var}
\end{align}
where $\lesssim$  hides some factors depending on $|E|$ and $|V|$ only. 

Let $\ell(G)$ denote the number of distinct communities in $\pi(G)$.
Combining~\eqref{eq:basic:var} with~\eqref{eq:cond:lambda/q} and $\bbP_{ij}[\ell(G)=\ell]\asymp K^{-(|V|-1-\ell)}$, we then get
\begin{align}
\mathrm{var}_{ij}(P_{G})\vee  \mathrm{var}_{\not ij}(P_{G})&\lesssim  \pa{n\lambda^r\over K}^{|V|-2}\ \sum_{\ell=1}^{|V|-1} K^{\ell-1} \bbE_{ij}\cro{K^{-|E^{\neq}|/r}\Big|\ell(G)=\ell}. \label{eq:V12:intro}
\end{align}
Taking for granted that the upper-bound in \eqref{eq:V12:intro} is valid up to constants or log factors, comparing~\eqref{eq:V12:intro} and~\eqref{eq:E12:intro},  a test based on $P_{G}$ 
will fulfill~\eqref{eq:test:success}, as soon as the sum in~\eqref{eq:V12:intro} remains small. This will occur in particular if the motif $G$ fulfills $|E^{\neq}| \geq r(\ell(G)-1)$ for any partition of the nodes into $\ell(G)$ communities. This leads to the following conjecture of~\cite{carpentier2025phase}.

\begin{conjecture}[Conjecture~1.4 in~\cite{carpentier2025phase}]\label{conjecture}
Let $G=(V,E)$ be a connected graph, with $(v_{1},v_{2})\notin E$, and such that 
 for any partition of $V$ in $\ell$ groups, with $v_{1},v_{2}$ in the same group,
the number $|E^{\neq}|$ of edges in $E$ between distincts groups fulfills 
\begin{equation}\label{eq:Eneq}
|E^{\neq}| \geq r(\ell-1)
\end{equation}
with $r:=|E|/(|V|-2)$. Then, if $q\asymp n^{-1/r}$  and $\lambda \geq_{\log} q^{1-\log_{n}(K)}$, we can recover the communities in the SBM with an algorithm based on  $P_{G}$, counting of the occurrence of the motif $G$.
\end{conjecture}

In the next section, for any rational $r>1$, 
\begin{itemize}
\item we construct a connected motif $G$ fulfilling the key condition~\eqref{eq:Eneq} ;
\item we prove that when  $q\asymp n^{-1/r}$  and~\eqref{eq:new_intro} hold, we can recover the communities in the SBM with an algorithm based on $P_{G}$.
\end{itemize}

The family of motifs $G$ is based on the blow-up of a cycle. Writing $r= \gamma + a  $ with $\gamma$ an integer larger than $1$ and $a$ a rational in $(0,1)$, the motif is constructed by considering a $\gamma$-blow-up of a cycle of length $\kappa$, that we link evenly to $i$ and $j$ with  $a\kappa\gamma$ edges ---called fasteners in the following. The length $\kappa$ of the cycle is chosen in such a way that $a\kappa\gamma$ is a large enough even integer. Let us describe with more details this construction.

\section{Main Result}\label{sec:main}

\subsection{Construction of the blow-up graph with fasteners}

Consider any rational number $a\in (0,1)$, any positive integer $\gamma$. Fix $\kappa\geq 3\vee (2a^{-1})$ an integer such that $a\kappa \gamma$ is an even integer. In the following, we construct the blow-up graph  with fasteners $G_{\kappa,\gamma,a}=(V,E)$.

Let $C_\kappa$ be a simple cycle on vertices $[\kappa]$. For each vertex $\omega\in [\kappa]$, define a \emph{layer}
\[
L_\omega := \{ v_{\omega,t}: t=1,\dots,\gamma\},
\]
consisting of $\gamma$ vertices that represent the ``blow-up'' of the cycle node $\omega$. The collection of all such vertices forms the \emph{cycle nodes} of the graph:
\[
V_{\mathrm{cyc}} := \bigcup_{\omega=1}^\kappa L_\omega, \qquad |V_{\mathrm{cyc}}| = \kappa \gamma.
\]
Between every two consecutive layers $L_\omega$ and $L_{\omega+1}$ (with the convention $L_{\kappa+1}=L_1$), we insert a complete bipartite graph:
\[
E_{\mathrm{cyc}} := \bigcup_{\omega=1}^\kappa \{(v_{\omega,t},v_{\omega+1,t'}) : 1\le t,t'\le \gamma\,\},
\]
so that every vertex in a layer $L_\omega$ is connected to all $\gamma$ vertices in $L_{\omega-1}$ and all $\gamma$ vertices in $L_{\omega+1}$. This produces a total of $|E_{\mathrm{cyc}}| = \kappa \gamma^2$ edges and the graph induced by $G_{\mathrm{cyc}}=(V_{\mathrm{cyc}},E_{\mathrm{cyc}})$ is $2\gamma$-regular. We refer to $G_{\mathrm{cyc}}$ as the cycle part of the graph $G$. 

We then introduce the two additional \emph{distinguished} vertices $v_1$ and $v_2$, and define the set of nodes $V := V_{\mathrm{cyc}} \cup \{v_1,v_2\}$. 
We connect $v_{1},v_{2}$ to $V_{\mathrm{cyc}}$ in such a way that $|E|=r(|V|-2)=(\gamma+a)(|V|-2)$. Accordingly, we connect $v_{1}$ or $v_{2}$ to each node in $V_{\mathrm{cyc}}$ with frequency $a$. 
We denote by 
 $V_{\mathrm{fst}}\subseteq V_{\mathrm{cyc}}$ the cyclic vertices to be connected to $v_1$ or $v_2$,  which we call the \emph{fastener nodes}. Informally, the set $V_{\mathrm{fst}}$ whose cardinality is equal to $a \kappa \gamma$ is defined such that the nodes $V_{\mathrm{fst}}$ are as evenly spread as possible. More formally, for each layer \(\omega=1,\dots,\kappa\), we define the integers $s_\omega$ by
\begin{equation}\label{eq:definition:Fastener}
s_0=0,\qquad s_\omega = \lfloor a \omega \gamma \rfloor - (s_{0}+\ldots+s_{\omega-1})\ . 
\end{equation}
Note, that $s_{\omega}\in \{\lfloor a \gamma\rfloor, \lceil a \gamma\rceil\}$. In layer $L_\omega$, exactly $s_\omega$ vertices are selected into $V_{\mathrm{fst}}$. By enumerating $V_{\mathrm{fst}}$ in the lexicographic order and alternatively assigning these nodes to two sets $V_{\mathrm{fst},1}$ or to $V_{\mathrm{fst},2}$,  we partition $V_{\mathrm{fst}}$ into two subsets of size $a\kappa\gamma/2$ --which is  an integer, as we have assumed $a\kappa \gamma$ to be an even integer.

Finally, we introduce the collection $E_{\mathrm{fst}}$ of \emph{fastener edges} by 
\[
E_{\mathrm{fst}} := \{\, (v_1,v) : v\in V_{\mathrm{fst},1} \,\}\;\cup\; \{\, (v_2,v) : v\in V_{\mathrm{fst},2} \,\}\enspace .
\]
The complete edge set of the graph is then $E := E_{\mathrm{cyc}} \cup E_{\mathrm{fst}}$. This construction is illustrated in Figure~\ref{fig:illBU}.

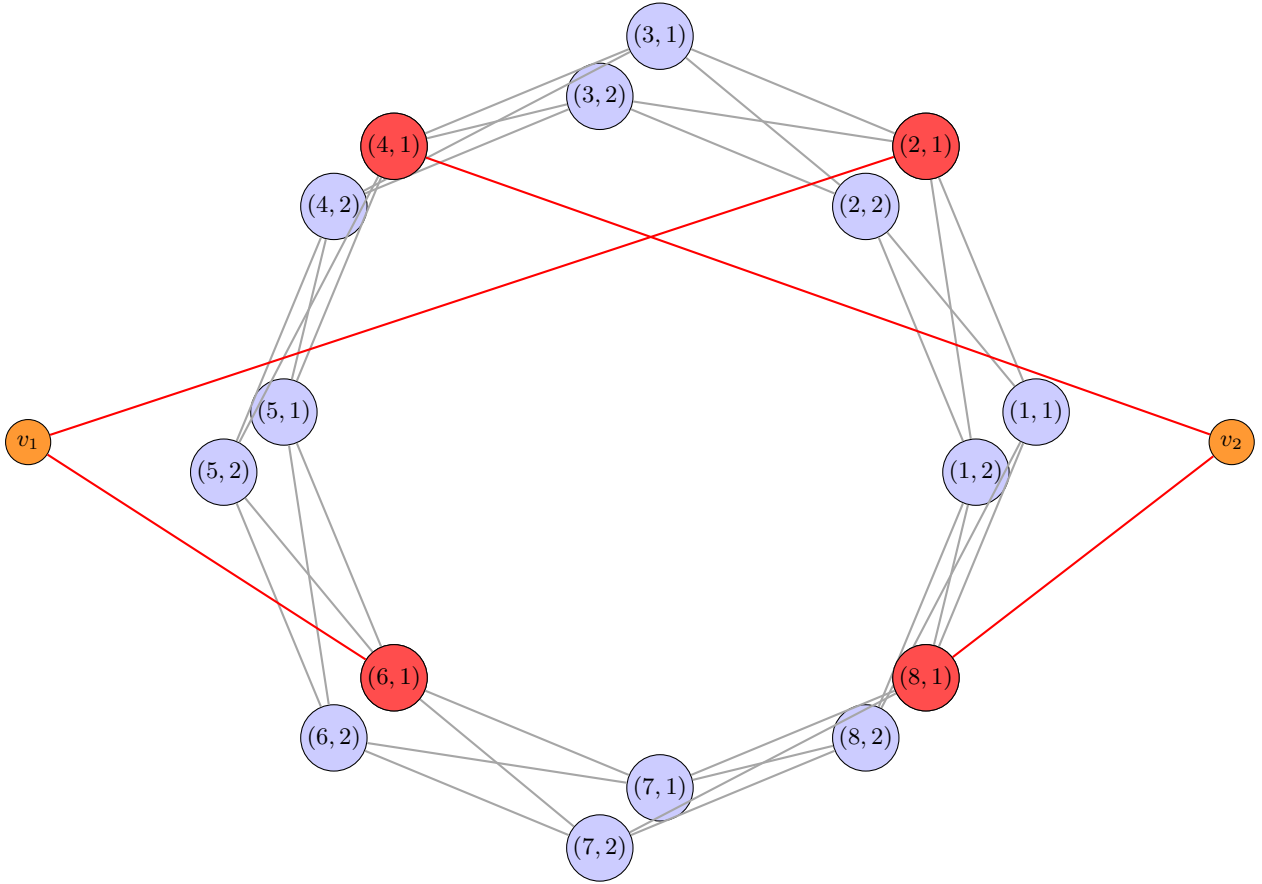
\begin{figure}[h]
\begin{center}
\begin{tikzpicture}[scale=2, every node/.style={circle, draw, minimum size=6mm, inner sep=1pt, font=\small}]

\def\kappa{6}
\def\gamma{2}
\def\radius{2.5}
\def\shift{0.2}

\foreach \w in {1,...,8}{
    \pgfmathsetmacro{\angle}{45*(\w-1)}
    \pgfmathsetmacro{\x}{\radius*cos(\angle)}
    \pgfmathsetmacro{\y}{\radius*sin(\angle)}
    \node[fill=blue!20] (n\w1) at ({\x+\shift}, {\y+\shift}) {$(\w,1)$};
    \node[fill=blue!20] (n\w2) at ({\x-\shift}, {\y-\shift}) {$(\w,2)$};
}

\foreach \w in {1,...,8}{
    \pgfmathtruncatemacro{\next}{mod(\w,8)+1}
    \foreach \t in {1,2}{
        \foreach \tp in {1,2}{
            \draw[thick,gray!70] (n\w\t) -- (n\next\tp);
        }
    }
}

\node[fill=orange!80] (v1) at (-4,0) {$v_1$};
\node[fill=orange!80] (v2) at (4,0) {$v_2$};

\foreach \w/\t in {2/1,4/1,6/1,8/1}{
    \node[fill=red!70] at (n\w\t) {$(\w,1)$};
}

\draw[thick, red] (v1) -- (n21);
\draw[thick, red] (v1) -- (n61);
\draw[thick, red] (v2) -- (n41);
\draw[thick, red] (v2) -- (n81);

\end{tikzpicture}
\end{center}
\caption{Illustration of the blow-up graph with fasteners $G_{\kappa,\gamma,a}$, in the case where $\gamma = 2, \kappa = 8, a=0.25$. The distinguished nodes $v_1,v_2$ are in orange, the {\it fastener nodes} (in $V_{\mathrm{fst}}$) are in red, the other cycle nodes are in blue. The {\it fastener edges} (in $E_{\mathrm{fst}}$) are in red, while the other edges are in gray.}\label{fig:illBU}
\end{figure}

Interestingly, we observe
\begin{align*}
|V| &= \kappa \gamma + 2 \ ; |E| = \kappa \gamma^2 + a \kappa \gamma\ . 
\end{align*}
so that the graph $G_{\kappa,\gamma, a}$ satisfies
\[
\frac{|E|}{|V|-2}= \gamma + a\ . 
\]

 The following proposition ensures that the motif $G_{\kappa,\gamma,a}$ fulfills the Condition~\eqref{eq:Eneq} required in Conjecture~\ref{conjecture}. We recall that  $\kappa\geq (3\vee 2/a)$ is such $a \kappa\gamma$ is an even integer.

\begin{proposition}\label{prop:key:blow-up}
Fix any positive integer $I\leq |V_{\mathrm{cyc}}|$ and consider any partition of $V$ into $I+1$ communities such that both $v_1$ and $v_2$ are in the same community. Define $E^{\neq}\subset E$ as the set of edges between nodes of distinct communities. Then, we have
\[
|E^{\neq}|\ge I (\gamma +a)\enspace .
\]
\end{proposition}

\subsection{Counting Blow-up Motifs}

Consider a blow-up graph with fasteners $G=G_{\kappa,\gamma , a }$. For $i< j$, we remind that $\Pi_{i,j}$ is the set of injections $\pi:V\to \ac{1\ldots,n}$ such that $\pi(v_{1})=i$ and $\pi(v_{2})=j$, and we set
\begin{equation}\label{eq:def:R}
R_{ij}=\sum_{\pi \in \Pi_{i,j}} P_{G,\pi}(Y),\quad \text{with}\quad P_{G,\pi}(Y)=\prod_{(v,v')\in E}Y_{\pi(v),\pi(v')},
\end{equation}
where $Y$ is the ``centered'' adjacency matrix~\eqref{eq:definition:Y}.

In the following proposition, we control the mean and the variance of $R_{ij}$ both when $z_{i}=z_{j}$ and when $z_i\neq z_j$. We recall that the conditional probabilities
$\bbP_{ij}$ and $\bbP_{\not ij}$ are defined in~\eqref{eq:conditional:P}.

\begin{proposition}\label{prop:mean:variance:blow_up}
 Assume that $q\leq 1/4$, $q+2\lambda \leq 1$, and $n\geq 2\kappa\gamma +4$. We also assume that $a\kappa \gamma$ is an even integer and that $\kappa\geq 3\vee 2/a$. 
Let $i<j$ and let $R_{ij}$ be defined by \eqref{eq:def:R}.
We have 
\begin{align}
\bbE_{ij}\cro{R_{ij}}&={(n-2)!\over (n-\kappa\gamma-2)!} \left({\lambda^{ \gamma + a } \over K}\right)^{\kappa \gamma }, \label{eq:mean:12:prop:blow_up} \quad \quad 
\bbE_{\not ij}\cro{R_{ij}} = 0 \enspace . 
\end{align}
In addition, if for some $\rho>1$, 
\begin{align}
\left(\frac{\lambda^2}{2\bar q}  \right)^{\gamma + a} \geq \frac{2K^2(\kappa \gamma)^5 \rho}{n}\ ; \quad \quad    \lambda^{\gamma + a}\geq 
 \frac{2K(\kappa \gamma)^5 \rho}{n} \label{eq:cond1:blow_up}
 \end{align}
then, we have
\begin{align}
\mathrm{var}_{ij}(R_{ij})\bigvee  \mathrm{var}_{\not ij}(R_{ij})  & \leq \frac{4}{\rho} \cdot \bbE^2_{ij}\cro{R_{ij}} \label{eq:var:12:prop:blow_up}\enspace . 
\end{align}
\end{proposition}
As for the clique counting or self-avoiding path counting problems in~\cite{carpentier2025phase}, relying on $R_{ij}$ alone together with a Markov type bound is not sufficient. For this reason, we rely again on a Median-of-Means post-processing and we use the same notation as in the latter. In particular, we fix $\Lambda= 24\log(n)$ and we assume for simplicity that $(n-2)/\Lambda$ is an integer. Recall 
the  definition of $N=(n-2)/\Lambda+2$ and let $J^{(1)},\ldots,J^{(\Lambda)}$ be a partition of $[n]\setminus\ac{i,j}$ into $L$ disjoint parts. For $\ell=1,\ldots, \Lambda$, we define $\Pi_{i,j}^{(\ell)}$ has the set of injections $\pi:V\to \ac{i,j}\cup J^{(\ell)}$, such that $\pi(v_{1})=i$ and $\pi(v_{2})=j$. Then, as in
 (31) in~\cite{carpentier2025phase}
for clique counts, we introduce the blow-up count  $R_{ij}^{(\ell)}:= \sum_{\pi \in \Pi_{i,j}^{(\ell)}} P_{G,\pi}(Y)$, and we define $M_{ij}$ as a median of the set $\ac{R^{(1)}_{ij},\ldots,R^{(\Lambda)}_{ij}}$. We estimate  $x_{ij}={\bf 1}_{z_{i}=z_{j}}-{1\over K}$ by 
\begin{equation}\label{eq:hatx:blow-up}
\hat x_{ij}=\mathbf{1}\ac{M_{ij}> {(N-2)!\over 2(N-\kappa\gamma -2)!} \left({\lambda^{ \gamma + a } \over K}\right)^{\kappa \gamma }}-{1\over K},\quad \text{where}\quad N={n-2\over 24 \log(n)}+2.
\end{equation}

We can now state our main result.
\begin{theorem}\label{thm:blow-up}
There exists a numerical constant $c$ such that the following holds. 
Assume that $q\leq 1/4$, $q+2\lambda \leq 1$,  that $N:=2+(n-2)/(24\log(n))$ is an integer, and $N\geq 2\kappa \gamma + 4$, and that $\kappa \gamma a$ is an an even integer.  Provided that 
\begin{align*}
\left(\frac{\lambda^2}{2\bar q}  \right)^{\gamma + a} \geq c\frac{K^2(\kappa \gamma)^5 \log(n)}{n}\ ; \quad \quad    \lambda^{\gamma + a}\geq 
 c \frac{K(\kappa \gamma)^5\log(n) }{n}\enspace  , 
 \end{align*}
we have for $\hat x_{ij}$ defined by \eqref{eq:hatx:blow-up} that $\bbP\pa{\hat x_{ij}=x_{ij}}\geq 1-n^{-3}$. 
\end{theorem}

\begin{proof}[Proof of Theorem~\ref{thm:blow-up}]
The proof follows exactly the same lines as that of Theorem~
2.5 in~\cite{carpentier2025phase}
to the difference that we build upon Proposition~\ref{prop:mean:variance:blow_up} instead of Proposition
2.4 in~\cite{carpentier2025phase}. We skip the details.
\end{proof}

\begin{corollary}\label{cor:blow-up}
Let  $\bar q=n^{-1/r}$ for some  $r= \gamma + \alpha /\beta $ where  $\gamma$, $\alpha $, and $\beta$ are positive integers with $\alpha<\beta$. Consider the blow-up graph with fasteners $G_{\kappa,\gamma,a }$ with $\kappa = 2\beta\gamma$  and $a=\alpha /\beta $. When $n\geq c_0 \beta\gamma^2  \log(n)$, $\lambda \leq 1-2q$, and 
\begin{equation}\label{eq:blow-up:final}
\lambda \geq w'_{r} \log^{1/r}(n) \pa{{\lambda \over K}+\bar q}^{1-\log_{n}(K)},
\end{equation}
with $w'_{r}$ depending only on $r$, then the  estimator $\hat{x}$ based on the number of blow-up motifs recovers the communities with probability at least $1-1/n$. 
\end{corollary}

If $q=n^{-1/r}$, with $r$ a rational number, recovering the communities above the threshold~\eqref{eq:new_intro} is feasible with a polynomial of fixed degree, only depending on $r$. 
When $r>1$ is not a rational number, we can still count blow-up motif for some rational number $\bar{r}$ close to $r$. Indeed, consider any $\epsilon < 1$. There exists an integer $\overline{\beta}\leq 2/[\epsilon r^2]\vee 1$ and a rational number $\overline{r}= \overline{\gamma}+ \overline{\alpha}/\overline{\beta}$ such that $|\overline{r}-r|\leq \epsilon r^2/2 $. Here, we have $\bar \gamma = \lfloor r \rfloor$. Choosing $\overline{\kappa}= 2\overline{\beta}\overline{\gamma}$, we 
deduce from Theorem~\ref{thm:blow-up}, that  the  estimator $\hat{x}$ based on counting occurrence of blow-up motifs $G_{\bar \kappa,\bar \gamma, \bar \alpha/\bar \beta}$ recovers the communities with probability at least $1-1/n$ as long as 
\[
\lambda \geq w''_{r} \epsilon^{-5} \log^{1/r}(n) \pa{{\lambda \over K}+\bar q}^{1-\log_{n}(K)} n^{\epsilon}\ . 
\]
The corresponding polynomial has a degree of the order $\epsilon r$. In particular, if we take $\epsilon=\log(\log(n))/\log(n)$, we establish that, for any irrational $r$, it is possible to recover the communities above the threshold~\eqref{eq:new_intro} with a polynomial of degree $O(\log(n)/\log(\log(n)))$. Whether this polynomial can be computed (or well-approximated) in polynomial-time remains an open question.

\section{Proof of Proposition~\ref{prop:key:blow-up}}\label{sec:key:blow-up}

In this proof, we need additional notation. Let $V_{\mathrm{cyc},0}$ denote the set of cycle nodes  having the same community as $v_1$ and $v_2$. For $k=1,\dots,I$ let $V_{\mathrm{cyc},k}$ be the set of cycle nodes of community $k$. For short, we write 
$V_{\mathrm{cyc},\neq 0}=\cup_{k=1}^IV_{\mathrm{cyc},k}$. 

We consider separately cycle edges and fastener edges. For $k=0,\ldots, I$, we write $E^{\uparrow}_{\mathrm{cyc},k}\subset E_{\mathrm{cyc}}$ for the collection of cycle edges that are incident to a node in $V_{\mathrm{cyc},k}$ and a node from a different community --- these edges will be called \textit{boundary} edges. For $k=0,\ldots, I$, we define $V_{\mathrm{fst},k} := V_{\mathrm{cyc},k} \cap V_{\mathrm{fst}}$, the subset of fastener nodes within community $k$. Note that $|V_{\mathrm{fst},k}|$ also corresponds to the number of fastener edges within community $k$. For short, we also write $V_{\mathrm{fst},\neq 0}=\cup_{k=1}^I V_{\mathrm{fst},k}$.

\bigskip

First, we consider the specific  and trivial case where $I=1$ and $V_{\mathrm{cyc},1}=V_{\mathrm{cyc}}$. Then, we obviously have 
\[
|E^{\neq}|= |V_{\mathrm{fst}}|=a\kappa\gamma \geq \gamma+1 \geq (\gamma+a)I\ , 
\] 
since $\kappa\geq 2/a$.  We assume henceforth $V_{\mathrm{cyc},k}\neq V_{\mathrm{cyc}}$ for any $k$.

The proof mainly divides into the two following lemmas. The first one states that the boundary edges of any community that does not contain $v_1$ or $v_2$ --- namely the sets of edges $E_{\mathrm{cyc},k}^{\uparrow}$ for $k \in \{1, \ldots, I\}$ --- is of size at least $2\gamma$.

\begin{lemma}\label{lem:Ecyc_k}
If $|V_{\mathrm{cyc},k}| < \kappa \gamma$, i.e. $V_{\mathrm{cyc},k}\neq V_{\mathrm{cyc}}$, we have 
\[
|E_{\mathrm{cyc},k}^{\uparrow}|\geq 2\gamma\enspace . 
\]
\end{lemma}
The second lemma lower bounds the sum of twice the number of fasteners $V_{\mathrm{fst},\neq 0}$ that do not belong to the same community as $v_1$ or $v_2$, plus the number of boundary edges for the community of $v_1,v_2$, namely $E_{\mathrm{cyc},0}^{\uparrow}$. If the nodes $V_{\mathrm{cyc},\neq 0}$ were sampled uniformly, then $|V_{\mathrm{fst},\neq 0}|$ would be of the order of $a|V_{\mathrm{cyc},\neq 0}|$. However, the number of fasteners in $V_{\mathrm{fst},\neq 0}$ can be much smaller for unfavorable configurations such that fastener nodes are in the same community as $v_1$ or $v_2$. Nevertheless, if this happens, then the number of boundary edges in $E_{\mathrm{cyc},0}^{\uparrow}$ must be high accordingly to compensate this drop in the number of fasteners in $V_{\mathrm{fst},\neq 0}$. 
\begin{lemma}\label{lem:eq:out_0}
    We have 
 \begin{equation}\label{eq:objective_fastener}
|E_{\mathrm{cyc},0}^{\uparrow}|+ 2 |V_{\mathrm{fst},\neq 0}| \geq 2 a |V_{\mathrm{cyc},\neq 0}| \geq 2a I\enspace . 
\end{equation}
\end{lemma}

Let us explain how Proposition~\ref{prop:key:blow-up} follows from Lemmas~\ref{prop:key:blow-up} and~\ref{lem:Ecyc_k}.
First, we observe that  an edge between two communities is either a fastener edge and is therefore incident to $V_{\mathrm{fst}, \neq 0}$ or is a cycle edge and therefore both arises in $E_{\mathrm{cyc},k}^{\uparrow}$ and $E_{\mathrm{cyc},k'}^{\uparrow}$ for $k\neq k'$ being the two communities to which the nodes of the edge belong.
Hence, we have the following decomposition
\[
2|E^{\neq}|\geq \sum_{k=0}^{I} |E_{\mathrm{cyc},k}^{\uparrow}| + 2|V_{\mathrm{fst}, \neq 0}|\ ,
\]
and  Lemmas~\ref{lem:Ecyc_k} and~\ref{lem:eq:out_0} ensure  that 
\[
2|E^{\neq}|\geq 2\gamma I + 2a I \ , 
\]
which concludes the proof of the Proposition~\ref{prop:key:blow-up}.

\begin{proof}[Proof of Lemma~\ref{lem:Ecyc_k}]
We consider three cases, 
(i) at least a layer does not intersect $V_{\mathrm{cyc},k}$; (ii)  each layer intersects  $V_{\mathrm{cyc},k}$ and $|V_{\mathrm{cyc},k}| \leq (\kappa-1) \gamma$; (iii)   each layer intersects  $V_{\mathrm{cyc},k}$ and $|V_{\mathrm{cyc},k}| > (\kappa-1) \gamma$. 

\medskip

\noindent
\underline{\bf Case (i): at least a layer does not intersect $V_{\mathrm{cyc},k}$.} In this case there exist two layers $i_1$ and $i_2$ (not necessarily distinct) such that 
\[
L_{i_1}\cap V_{\mathrm{cyc},k}\neq \emptyset;\quad  L_{i_1-1}\cap V_{\mathrm{cyc},k}= \emptyset;\quad  L_{i_2}\cap V_{\mathrm{cyc},k}\neq \emptyset;\quad  L_{i_2+1}\cap V_{\mathrm{cyc},k}= \emptyset ,
\]
with the convention that $L_0=L_{\kappa}$ and $L_{\kappa+1}= L_1$. 
There are $\gamma$ edges between $L_{i_1-1}$ and one node of $V_{\mathrm{cyc},k}$ in $L_{i_1}$. Similarly, there are $\gamma$ edges between $L_{i_2+1}$ and one node of $V_{\mathrm{cyc},k}$ in $L_{i_2}$. Since $\kappa \geq 3$, these edges are distinct and we deduce $|E_{\mathrm{cyc},k}^{\uparrow}|\geq 2\gamma$. 

\medskip

\noindent
\underline{\bf Case (ii): each layer intersects  $V_{\mathrm{cyc},k}$, with $|V_{\mathrm{cyc},k}| \leq (\kappa-1) \gamma$.} Since in this case each layer intersects $V_{\mathrm{cyc},k}$, it follows from the definition of the graph that each node $v$ in $V_{\mathrm{cyc}}\setminus V_{\mathrm{cyc},k}$ is connected to $V_{\mathrm{cyc},k}$ by at least two edges corresponding to nodes of $V_{\mathrm{cyc},k}$ in the previous layer and in the following layer of that of $v$. As a consequence, 
\[
|E_{\mathrm{cyc},k}^{\uparrow}| \geq 2|V_{\mathrm{cyc}}\setminus V_{\mathrm{cyc},k}|= 2[\kappa \gamma- |V_{\mathrm{cyc},k}|]\geq 2\gamma\ , 
\]
as soon as $|V_{\mathrm{cyc},k}| \leq (\kappa-1) \gamma$.

\medskip

\noindent
\underline{\bf Case (iii): each layer intersects  $V_{\mathrm{cyc},k}$, with $|V_{\mathrm{cyc},k}| > (\kappa-1) \gamma$.}
It remains to consider the case where $|V_{\mathrm{cyc},k}|\in [(\kappa-1)\gamma+1, \kappa\gamma-1]$ which is non empty only if $\gamma>1$. Then, we consider the complementary set $V'= V_{\mathrm{cyc}}\setminus V_{\mathrm{cyc},k}$ which satisfies $|V'|\leq\gamma-1\leq \kappa(\gamma-1)$. Since $V_{\mathrm{cyc},k}$ and $V'$ have a symmetric role for counting $|E_{\mathrm{cyc},k}^{\uparrow}|$, 
 we can apply the same arguments as in Case~(i) and~(ii), switching the role of $V_{\mathrm{cyc},k}$ and $V'$. Accordingly,  $|E_{\mathrm{cyc},k}^{\uparrow}|\geq 2\gamma$, and the proof of Lemma~\ref{lem:Ecyc_k} is complete.

\end{proof}

\begin{proof}[Proof of Lemma~\ref{lem:eq:out_0}]

Write $\mathcal{L}_0= \{\omega\in [\kappa]: L_{\omega}\cap V_{\mathrm{cyc},0}\neq \emptyset\}$ the collection of layers that intersect $V_{\mathrm{cyc},0}$. To start with, we consider three cases (a) $\mathcal{L}_0=\emptyset$, (b) $\mathcal{L}_0=[\kappa]$, and (c) $|\mathcal{L}_0|\in [2;\kappa-1]$.

\noindent
\underline{\bf Case (a): $\mathcal{L}_0=\emptyset$.} In this case no fastener node belongs to the community of $v_1, v_2$, therefore we have 
\[
|V_{\mathrm{fst},\neq 0}|=  |V_{\mathrm{fst}}|= a |V_{\mathrm{cyc}}|\geq a  |V_{\mathrm{cyc},\neq 0}| \ , 
\]
and~\eqref{eq:objective_fastener} holds.

\medskip 

\noindent
\underline{\bf Case (b): $\mathcal{L}_0=[\kappa]$.} In this case we observe as in the proof of Lemma~\ref{lem:Ecyc_k} (case (ii)), that any node $v$ in $V_{\mathrm{cyc}}\setminus V_{\mathrm{cyc},0}$ is connected to two nodes in $V_{\mathrm{cyc},0}$, one in the previous layer and one in the following layer. Thus, we get 
\[
|E_{\mathrm{cyc},0}^{\uparrow}|\geq 2 \sum_{k=1}^I |V_{\mathrm{cyc},k}| =  2|V_{\mathrm{cyc},\neq 0}| \geq 2a|V_{\mathrm{cyc},\neq 0}|\ , 
\]
and~\eqref{eq:objective_fastener} holds.

\medskip 

\noindent
\underline{\bf Case (c): $|\mathcal{L}_0|\in [2;\kappa-1]$.} This is the intermediary case between cases (a) and (b). Consider $[\kappa]$ as the nodes of a cycle graph and $\mathcal{L}_0$ as a subset of nodes of this graph. Write $\#\mathrm{CC}(\mathcal L_0)\geq 1$ for the number of connected components of the subgraph induced by $\mathcal{L}_0$. 
\smallskip

\noindent
\underline{Step 1: lower bound of $|E_{\mathrm{cyc},0}^{\uparrow}|$.} We partition $\mathcal{L}_0$ into $\mathcal{L}_{0,\mathrm{in}}$, $\mathcal{L}_{0,\mathrm{bd}}$, and $\mathcal{L}_{0,\mathrm{sg}}$ where $\mathcal{L}_{0,\mathrm{in}}$ are internal layers of $\mathcal{L}_0$ that is, they are not at the boundary of a connected component of $\mathcal{L}_0$, $\mathcal{L}_{0,\mathrm{sg}}$ correspond to connected components that are singletons, and $\mathcal{L}_{0,\mathrm{bd}}$ stands for all the the remaining boundary layers. 

Fix any $\omega\in \mathcal{L}_{0,\mathrm{in}}$ and any node $v\in L_\omega\setminus V_{\mathrm{cyc},0}$. This node $v$ is connected by at least two edges to $V_{\mathrm{cyc},0}$ (at least one node in the previous layer and one node in the following layer). Fix any $\omega\in\mathcal{L}_{0,\mathrm{bd}}$. For any node $v\in L_\omega\setminus V_{\mathrm{cyc},0}$, we also observe that $v$ is connected by at least one edge to $V_{\mathrm{cyc},0}$. Also, for any $v\in L_{\omega}\cap V_{\mathrm{cyc},0}$, there are $\gamma$ edges between $v$ and nodes in a layer that does not belong to $\mathcal{L}_0$. Finally, we consider any $\omega\in \mathcal{L}_{0,\mathrm{sg}}$  and any $v\in L_{\omega}\cap V_{\mathrm{cyc},0}$. This node is connected to $2\gamma$ nodes outside the layers in $\mathcal{L}_0$. Gathering these bounds, we get
\begin{align}\nonumber
|E_{\mathrm{cyc},0}^{\uparrow}|&\geq 2 \sum_{\omega\in\mathcal{L}_{0,\mathrm{in}}}|L_{\omega}\setminus V_{\mathrm{cyc},0}|+ \sum_{\omega\in\mathcal{L}_{0,\mathrm{bd}}}[|L_{\omega}\setminus V_{\mathrm{cyc},0}| + \gamma]  +  2\sum_{\omega\in\mathcal{L}_{0,\mathrm{sg}}}\gamma  \\
& \geq 2 \sum_{\omega\in\mathcal{L}_{0}}|L_{\omega}\setminus V_{\mathrm{cyc},0}| + 2\#\mathrm{CC}(\mathcal L_0)\  , \label{eq:lower:E:cyc,0,up}
\end{align}
where we used that, for any $\omega\in \mathcal{L}_{0}$, $|L_{\omega}\setminus V_{\mathrm{cyc},0}|\leq \gamma - 1$, and where we recall that $\#\mathrm{CC}(\mathcal L_0)$ is the number of connected components of $\mathcal{L}_0$. 

\smallskip

\noindent
\underline{Step 2: lower bound of $|V_{\mathrm{fst},\neq 0}|$.} We now focus on the set $[\kappa]\setminus \mathcal{L}_0$ of layers that do not intersect $V_{\mathrm{cyc},0}$. We start from 
\[
|V_{\mathrm{fst},\neq 0}|\geq \sum_{\omega\in [\kappa]\setminus \mathcal{L}_0} |V_{\mathrm{fst}}\cap L_{\omega}|.
\]
Although $|V_{\mathrm{fst}}\cap L_{\omega}|\in [\lfloor a\gamma\rfloor, \lceil a\gamma\rceil]$ for each $\omega$, we have to be slightly more careful to control $|V_{\mathrm{fst},\neq 0}|$. The set $[\kappa]\setminus \mathcal{L}_0$ decomposes into $\#\mathrm{CC}(\mathcal L_0)$ connected components, since the set $\mathcal L_0$ decomposes in $\#\mathrm{CC}(\mathcal L_0)$ connected components. We write $r_1,\ldots, r_{\#\mathrm{CC}(\mathcal L_0)}$ for their respective number of layers. Into the $i$-th connected components, there are at least $\lfloor a r_i\gamma\rfloor$ fasteners. From this and from $\sum_i r_i \geq \kappa - |\mathcal L_0|$, we deduce that 
\begin{align}\nonumber
|V_{\mathrm{fst},\neq 0}|&\geq \sum_{i=1}^{\#\mathrm{CC}(\mathcal L_0)} \lfloor a r_i \gamma\rfloor \geq a \left[\kappa - |\mathcal{L}_0|\right]\gamma - \#\mathrm{CC}(\mathcal L_0)\\
&= a \sum_{\omega\in [\kappa]\setminus \mathcal L_0}|L_{\omega}\setminus V_{\mathrm{cyc},0}|   - \#\mathrm{CC}(\mathcal L_0)\enspace . \label{eq:fastener_open}
\end{align}
Combining~\eqref{eq:lower:E:cyc,0,up} and \eqref{eq:fastener_open}, we obtain 
\begin{align*}
|E_{\mathrm{cyc},0}^{\uparrow}|+ 2 |V_{\mathrm{fst},\neq 0}| &\geq  2a \left[\sum_{\omega\in [\kappa]}|L_{\omega}\setminus V_{\mathrm{cyc},0}| \right] =  2a \left[\sum_{\omega\in [\kappa]}|L_{\omega}\cap V_{\mathrm{cyc},\neq 0}| \right]\\
& \geq 2a |V_{\mathrm{cyc},\neq 0}|\ ,
\end{align*}
which concludes the proof of Lemma~\ref{lem:eq:out_0}.
\end{proof}

\section{Proof of Proposition~\ref{prop:mean:variance:blow_up}}\label{sec:mean:variance:blow_up}

This proposition is a straightforward consequence of the following result. 

\begin{proposition}\label{lem:blow_up}
Assume that $q\leq 1/4$, $q+2\lambda \leq 1$, $n\geq 2\kappa\gamma +4$, and set $\bar p=\bar q+\lambda(1-2q)$. We also assume that $a\kappa \gamma$ is an even integer and $\kappa\geq 3\vee 2/a$. 
Then, for any $1\leq i< j\leq n$, we have 
\begin{align}
\bbE_{ij}\cro{R_{ij}}&={(n-2)!\over (n-\kappa\gamma-2)!} \left({\lambda^{ \gamma + a } \over K}\right)^{\kappa \gamma }, \label{eq:mean:12:blow-up}\\
\bbE_{\not ij}\cro{R_{ij}}&= 0, \label{eq:mean:not12:blow-up}\\
 \mathrm{var}_{\not ij}(R_{ij})\vee \mathrm{var}_{ij}(R_{ij}) & \leq \bbE^2_{12}[R_{12}](\kappa\gamma)^2 \Bigg[\frac{2(\kappa \gamma)^{3} K^2 } {n}  \left(\frac{\bar{q}}{\lambda^2}\vee \frac{\bar{q}}{\bar p } \right)^{\gamma + a}  + \frac{2(\kappa \gamma)^{3} K } {n} \left(\frac{\bar{p}}{\lambda^2}\vee 1 \right)^{\gamma + a} \nonumber  \\& \quad\quad\quad \quad \quad \, \quad   + \left[\frac{2(\kappa \gamma)^{3} K^2 } {n}   \left(\frac{\bar{q}}{\lambda^2}\vee \frac{\bar{q}}{\bar p } \right)^{\gamma + a}\right]^{\kappa \gamma}  + \left[\frac{2(\kappa \gamma)^{3} K } {n} \left(\frac{\bar{p}}{\lambda^2}\vee 1 \right)^{\gamma + a}\right]^{\kappa \gamma}   \Bigg] \ . 
\label{eq:var:12:blow-up}
\end{align}
\end{proposition}

\subsection{Notation and preliminaries}

To control the expectation and the variance of $R_{ij}$, we need to introduce some graph notation as in Appendix~A of~\cite{carpentier2025phase}.
For a motif $G=(V,E)$ and a labeling $\pi\in\Pi_{ij}$, we define the labeled graph $\pi(G)$ as the graph with node set $\ac{\pi(v):v\in V}$ and edge set $\ac{(\pi(v),\pi(v')):(v,v')\in E}$. Given two labelings $\pi^{(1)}$ and $\pi^{(2)}\in \Pi_{ij}$, the labeled merged graph $G_{\cup}=(V_{\cup},E_{\cup})$ is defined as the union of $\pi^{(1)}(G^{(1)})$ and $\pi^{(2)}(G^{(2)})$, with the convention that two same edges are merged into a single edge. Similarly,  we define the  intersection graph  $G_{\cap}=(V_{\cap},E_{\cap})$ and the  symmetric difference graph $G_{\Delta}=(V_{\Delta},E_{\Delta})$ so that $E_{\Delta}=E_{\cup}\setminus E_{\cap}$. Here, $V_{\cap}$ (resp. $V_{\Delta}$) is the set of nodes induced by the edges $E_{\cap}$ (resp. $E_{\Delta}$) so that $G_{\cap}$ (resp. $G_{\Delta}$) does not contain any isolated node.

The two following expressions are the starting point of our analysis. They easily derive from the definition~\eqref{eq:definition:Y} of $Y$:
\begin{align}
 \bbE\cro{P_{G, \pi}}&=    \bbE\cro{\prod_{(v,v')\in E }Y_{\pi(v)\pi(v')}}   \label{eq:proof:ingredient:basis:mean}
    =   \bbE\cro{\prod_{(v,v')\in E} (\lambda \mathbf{1}_{z_{\pi(v)}=z_{\pi(v')}})},\\
   \bbE\cro{P_{G, \pi^{(1)}}P_{G, \pi^{(2)}}} 
   &   = \bbE\cro{\prod_{(i,j)\in E_{\Delta}}Y_{ij}\prod_{(i,j)\in E_{\cap}}Y_{ij}^2}   \label{eq:proof:ingredient:basis}
    =   \bbE\cro{\prod_{(i,j)\in E_{\Delta}} (\lambda \mathbf{1}_{z_{i}=z_{j}}) \prod_{(i,j)\in E_{\cap}}\bar q \pa{\bar p\over \bar q}^{ \mathbf{1}_{z_{i}=z_{j}}}}.
\end{align}

Without loss of generality,  we assume in the following that $(i,j)=(1,2)$.

\subsection{Results under the distribution $\bbP_{12}$}

In what follows, we consider the graph $G := G_{\kappa,\gamma,a}=(V,E)$.  We will study respectively the expectation and variance of the associated polynomial, both under $\bbP_{12}$ and under $\bbP_{\not 12}$.

\paragraph{Mean under $\bbP_{12}$.}
Since the graph $G$ is connected, we have from Equation~\eqref{eq:proof:ingredient:basis:mean}
\begin{equation}\label{eq:mean:12:pi:blow-up}
\bbE_{12}\cro{P_{G,\pi}(Y)}= \lambda^{|E|} \,\bbE_{ 12}\cro{\prod_{(i,j)\in \pi(E)}\mathbf 1_{z_{i}=z_{j}}}= \frac{1}{K^{|V|-2}} \lambda^{|E|} = \frac{1}{K^{\kappa \gamma}} \lambda^{(\gamma+a)\kappa\gamma},
\end{equation}
since $z_{1}= z_{2}$ a.s. under  $\bbP_{12}$. The identity~\eqref{eq:mean:12:blow-up} follows.

\paragraph{Variance under $\bbP_{12}$.}
Let us turn to proving~\eqref{eq:var:12:blow-up}, which is the main part of the proof. Fix $\pi^{(1)},\pi^{(2)}\in\Pi_{12}$. 
We start by controlling $\bbE_{12}\cro{P_{G,\pi^{(1)}}P_{G,\pi^{(2)}}}$. Let $u\in\{0,\dots,|V_{\mathrm{cyc}}|\}$ be the integer so that
\[
|\mathrm{range}(\pi^{(1)})\cap\mathrm{range}(\pi^{(2)})| = 2+u\enspace.
\]
Recall that  $V_{\cap},V_{\Delta},E_{\cap},E_{\Delta}$ stand for the common and symmetric-difference vertex- and edge-sets. 
Following \eqref{eq:proof:ingredient:basis}, we have 
\[
\bbE_{12}\cro{P_{G,\pi^{(1)}}P_{G,\pi^{(2)}}}
=\lambda^{|E_{\Delta}|} \bar q^{|E_{\cap}|} \,
\bbE_{12}\cro{\prod_{(i,j)\in E_{\Delta}}\mathbf 1_{z_{i}=z_{j}}
\prod_{\substack{(i,j)\in E_{\cap}\\ z_{i}=z_{j}}}{\bar p\over \bar q}}.
\]

\noindent\underline{Case $u=0$.} Then, we have $E_{\cap}=\emptyset$, $|V_{\Delta}|=|V_{\cup}|=2\kappa \gamma+2$ and $|E_{\Delta}|=2\kappa \gamma(\gamma+a)$. Therefore
\begin{align}\nonumber
\bbE_{12}\cro{P_{G,\pi^{(1)}}P_{G,\pi^{(2)}}}
&= \lambda^{2\kappa \gamma(\gamma+a)} \,\bbP_{12}\cro{z_i=z_1\ \text{for }i\in V_{\Delta}} \\
& =  \lambda^{2\kappa \gamma(\gamma+a)} K^{-2\kappa\gamma}  \nonumber \\ 
& = \bbE_{12}\cro{P_{G,\pi^{(1)}}}\bbE_{12}\cro{P_{G,\pi^{(2)}}}\enspace , \label{eq:upper_cross_u:0:blow-up}
\end{align}
by~\eqref{eq:mean:12:pi:blow-up}.

\noindent\underline{Case $1\le u\le \kappa \gamma$.} In this case $|V_{\cup}|=2\kappa \gamma+2-u$ and there are $\kappa \gamma -u$ vertices that appear only in $\pi^{(2)}(V)\setminus\pi^{(1)}(V)$. All these vertices belong to $G_{\Delta}$ and each such vertex is connected (within $G_\Delta$) to at least one vertex in the intersection $\pi^{(1)}(V)\cap\pi^{(2)}(V)$ because  $G$ is connected. Let $E_{\Delta}^{(2)}\subset E_{\Delta}$ be the set of edges of $G_{\Delta}$ that arise from $\pi^{(2)}[G]$. In particular, $E_{\Delta}^{(2)}$ contains all the edges that have at least one endpoint in $\pi^{(2)}(V)\setminus\pi^{(1)}(V)$. Conditionally on the communities  of the vertices of $\pi^{(1)}(V)$, the product of indicators on $E_{\Delta}^{(2)}$ enforces the community of the  $\kappa \gamma-u$ vertices in $\pi^{(2)}(V)\setminus\pi^{(1)}(V)$. Hence,  we have 
\[
\bbE_{12}\left[\prod_{(i,j)\in E^{(2)}_{\Delta}} \mathbf 1_{z_{i}=z_{j}} \,\Bigg|\, z_{\pi^{(1)}(V)}\right] \le \frac{1}{K^{\kappa \gamma-u}}\qquad\text{a.s.}
\]
Therefore (conditioning and removing $E^{(2)}_\Delta$), we get 
\[
\bbE_{12}\cro{P_{G,\pi^{(1)}}P_{G,\pi^{(2)}}}
\le \lambda^{|E_{\Delta}|} \bar q^{|E_{\cap}|}\frac{1}{K^{\kappa \gamma-u}}
 \bbE_{12}\cro{\prod_{(i,j)\in E_{\Delta}\setminus  E^{(2)}_{\Delta} }\mathbf 1_{z_{i}=z_{j}}\prod_{\substack{(i,j)\in E_{\cap}\\ z_{i}=z_{j}}}{\bar p\over \bar q}}.
\]
Now, define $E^{(1)}_{\Delta}:= E_{\Delta}\setminus E^{(2)}_{\Delta}$. Observe that  $E^{(1)}_{\Delta}$ together with $E_{\cap}$ form a partition of the edges of the labeled copy of $G$ coming from $\pi^{(1)}$. Given the community assignments $(z_i)_{i\in \pi^{(1)}(V)}$, denote $\ell(z)$ the number of distinct communities in $\pi^{(1)}(V)$, and $|E^{\neq}|$ the number of edges in the graph between distinct communities. We have
\begin{align}
\bbE_{12}\cro{P_{G,\pi^{(1)}}P_{G,\pi^{(2)}}} \nonumber 
&\le \lambda^{|E_{\Delta}|}\bar p^{|E_{\cap}|}\frac{1}{K^{\kappa \gamma-u}}
\;\bbE_{12}\cro{\left({\bar q \over \bar p}\right)^{|E^{\neq}|} \prod_{(i,j)\in E^{(1)}_{\Delta}}\mathbf 1_{z_{i}=z_{j}}}\\ 
&\le \lambda^{|E_{\Delta}|}\bar p^{|E_{\cap}|}\frac{1}{K^{\kappa \gamma-u}}
\;\bbE_{12}\cro{\left({\bar q \over \bar p}\right)^{(\ell(z)-1)(\gamma + a)} \prod_{(i,j)\in E^{(1)}_{\Delta}}\mathbf 1_{z_{i}=z_{j}}} ,\label{eq:up_E_12:blow-up}
\end{align}
where, importantly, we apply Proposition~\ref{prop:key:blow-up} in the second line. Consider the  graph $\overline{\mathcal{G}}^{(1)}_{\Delta}=(\pi^{(1)}(V),E^{(1)}_{\Delta})$.  In Equation~\eqref{eq:up_E_12:blow-up} above, the condition $\prod_{(i,j)\in E^{(1)}_{\Delta}}\mathbf 1_{z_{i}=z_{j}}=1$ enforces that the nodes of each of the connected components of $\overline{\mathcal{G}}^{(1)}_{\Delta}$ are in the same community. Let $\mathrm{cc}$ denote the number of connected components of $\overline{\mathcal{G}}^{(1)}_{\Delta}$ when we remove those containing $\pi^{(1)}(v_1)=1$ or  $\pi^{(1)}(v_2)=2$. Then, conditionally to $\prod_{(i,j)\in E^{(1)}_{\Delta}}\mathbf 1_{z_{i}=z_{j}}=1$, we have $(\ell(z)-1)\in [0,\mathrm{cc}]$ and for any $x\in [0,\mathrm{cc}]$, we have 
\begin{equation}\label{eq:l(z)}
\mathbb{P}_{12}\cro{\ell(z)-1= x\,\Bigg|\, \prod_{(i,j)\in E^{(1)}_{\Delta}}\mathbf 1_{z_{i}=z_{j}}=1}\leq \mathrm{cc}^{\mathrm{cc}} K^{-(\mathrm{cc}-x)} \ . 
\end{equation}
Since the condition $\prod_{(i,j)\in E^{(1)}_{\Delta}}\mathbf 1_{z_{i}=z_{j}}=1$ enforces that the nodes of each of the connected components of $\overline{\mathcal{G}}^{(1)}_{\Delta}$ are in the same community, we get
$$\mathbb{P}_{12}\cro{\prod_{(i,j)\in E^{(1)}_{\Delta}}\mathbf 1_{z_{i}=z_{j}}=1}\leq {1\over K^{\kappa\gamma+1-cc-1}}={1\over K^{\kappa\gamma-cc}}$$
and, combining with \eqref{eq:l(z)}
\begin{align*}
    \bbE_{12}\cro{\mathbf 1_{l(z)-1 = x}\left({\bar q \over \bar p}\right)^{(\ell(z)-1)(\gamma + a)} \prod_{(i,j)\in E^{(1)}_{\Delta}}\mathbf 1_{z_{i}=z_{j}}} &\leq  \frac{1}{K^{\kappa \gamma -\mathrm{cc}}}  \cdot \mathrm{cc}^{\mathrm{cc}} K^{-(\mathrm{cc}-x)} \left(\frac{\bar q }{\bar p}\right)^{x (\gamma +a)}\\ 
    &\quad\quad = \frac{1}{K^{\kappa \gamma }}  \cdot \mathrm{cc}^{\mathrm{cc}} K^{x} \left(\frac{\bar q }{\bar p}\right)^{x (\gamma +a)}\enspace .
\end{align*}    
Since $G$ is connected,  each connected component of $\overline{\mathcal{G}}^{(1)}_{\Delta}$ must contain at least one node in $\pi^{(1)}(V)\cap \pi^{(2)}(V)$. As a consequence, we have $\mathrm{cc}\leq u$. Combining Equation~\eqref{eq:up_E_12:blow-up} with the last equation, and using  $|E_{\Delta}|+ 2|E_{\cap}|= 2|E|$, we get 
\begin{align*}
\bbE_{12}\cro{P_{G,\pi^{(1)}}P_{G,\pi^{(2)}}} 
  & \leq \lambda^{|E_{\Delta}|}\bar p^{|E_{\cap}|}\frac{1}{K^{2\kappa \gamma-u}} 
  u^u \sum_{x = 0 }^{u} \left(K \left(\frac{\bar q }{\bar p}\right)^{\gamma +a}\right)^{x }\ \\
  & \leq \lambda^{2|E|}\left(\frac{\bar p}{\lambda^2}\right)^{|E_{\cap}|}\frac{1}{K^{2\kappa \gamma-u}} 
  u^u \sum_{x = 0 }^{u} \left(K \left(\frac{\bar q }{\bar p}\right)^{\gamma +a}\right)^{x }\ .
\end{align*}
\begin{lemma}\label{lem:cap}
We have 
\[
|E_{\cap}|\leq (\gamma + a ) u.
\]
\end{lemma}
\begin{proof}[Proof of Lemma~\ref{lem:cap}]
This result is a consequence of Proposition~\ref{prop:key:blow-up}. Indeed, consider a community assignment such that the $2+u$ nodes in $\pi^{(1)}(V)\cap \pi^{(2)}(V)$ belong to same community, whereas the $\kappa \gamma-u$ other nodes in $\pi^{(1)}(V)\setminus \pi^{(2)}(V)$ belong to distinct communities. Hence, for that assignment, we have $\kappa\gamma - u + 1$ communities and $|E^{\neq}|= |E|-|E_{\cap}|$. Then, Proposition~\ref{prop:key:blow-up} states that 
\[
|E|-|E_{\cap}|\geq (\gamma+a)[\kappa\gamma -u]\  .
\]
Since $|E|=(\gamma+a )\kappa \gamma$, Lemma~\ref{lem:cap} holds. 
\end{proof}

We deduce from the previous lemma that 
\begin{align}\nonumber
\bbE_{12}\cro{P_{G,\pi^{(1)}}P_{G,\pi^{(2)}}}&\leq \bbE^2_{12}\cro{P_{G,\pi}(Y)}  \left[K \left(\frac{\bar p}{\lambda^2}\vee 1\right)^{\gamma + a}\right]^u    u^u \sum_{x = 0 }^{u} \left(K \left(\frac{\bar q }{\bar p}\right)^{\gamma +a}\right)^{x }\\ 
&\leq \bbE^2_{12}\cro{P_{G,\pi}(Y)}  u^{u+1} \left[\left[K^2  \left(\frac{\bar{q}}{\lambda^2}\vee \frac{\bar{q}}{\bar p } \right)^{\gamma + a}\right]^u  + \left[K \left(\frac{\bar{p}}{\lambda^2}\vee 1 \right)^{\gamma + a}\right]^u   \right] \enspace . \label{eq:upper:covariance:u:blow-up}
\end{align}

\noindent\underline{Combining the terms.}
Combining \eqref{eq:upper_cross_u:0:blow-up} and  \eqref{eq:upper:covariance:u:blow-up}, we get 
\begin{align*}
\text{var}_{12}(R_{12}) &= \sum_{u=0}^{\kappa \gamma} \sum_{\substack{\pi^{(1)},\pi^{(2)}\in \Pi_{12}\\ |\text{range}(\pi^{(1)})\cap \text{range}(\pi^{(2)})|=2+u}}
\pa{\bbE_{12}\cro{P_{G,\pi^{(1)}}P_{G,\pi^{(2)}}}-\bbE_{12}\cro{P_{G,\pi^{(1)}}}\bbE_{12}\cro{P_{G,\pi^{(2)}}}}\\
& \leq   \sum_{u=1}^{\kappa \gamma}   \binom{\kappa \gamma}{u}^{2} {(n-2)!u!\over (n-2\kappa\gamma+u - 2)!}  \bbE_{12}\cro{P_{G,\pi^{(1)}}P_{G,\pi^{(2)}}}\\
&\leq \left(\frac{(n-2)!}{(n-\kappa \gamma-2)!}\right)^2\sum_{u=1}^{\kappa \gamma} (\kappa \gamma)^{2u}  \frac{1}{(n-\kappa \gamma-2)^u}\bbE_{12}\cro{P_{G,\pi^{(1)}}P_{G,\pi^{(2)}}} \\ 
&\leq \left(\frac{(n-2)!}{(n-\kappa \gamma-2)!}\right)^2 \bbE^2_{12}\cro{P_{G,\pi}(Y)}\sum_{u=1}^{\kappa\gamma } \frac{u^{u+1}(\kappa \gamma)^{2u} } {(n-\kappa \gamma-2)^u} \Bigg[\left[K^2  \left(\frac{\bar{q}}{\lambda^2}\vee \frac{\bar{q}}{\bar p } \right)^{\gamma + a}\right]^u  + \left[K \left(\frac{\bar{p}}{\lambda^2}\vee 1 \right)^{\gamma + a}\right]^u   \Bigg]\\
&\leq \bbE^2_{12}[R_{12}](\kappa\gamma)^2 \Bigg[\frac{2(\kappa \gamma)^{3} K^2 } {n}  \left(\frac{\bar{q}}{\lambda^2}\vee \frac{\bar{q}}{\bar p } \right)^{\gamma + a}  + \frac{2(\kappa \gamma)^{3} K } {n} \left(\frac{\bar{p}}{\lambda^2}\vee 1 \right)^{\gamma + a} \\& \quad\quad\quad \quad \quad \quad \quad \quad   + \left[\frac{2(\kappa \gamma)^{3} K^2 } {n}   \left(\frac{\bar{q}}{\lambda^2}\vee \frac{\bar{q}}{\bar p } \right)^{\gamma + a}\right]^{\kappa \gamma}  + \left[\frac{2(\kappa \gamma)^{3} K } {n} \left(\frac{\bar{p}}{\lambda^2}\vee 1 \right)^{\gamma + a}\right]^{\kappa \gamma}   \Bigg] \ . 
\end{align*}
This concludes the bound on $\mathrm{var}_{12}(R_{12})$ in Equation~\eqref{eq:var:12:blow-up}.

\subsection{Results under the distribution $\bbP_{\not 12}$}

\paragraph{Mean under $\bbP_{\not 12}$.}
We have from Equation~\eqref{eq:proof:ingredient:basis:mean}
\begin{equation*}
\bbE_{\not 12}\cro{P_{G,\pi}(Y)}= \lambda^{|E|} \,\bbE_{\not 12}\cro{\prod_{(i,j)\in \pi(E)}\mathbf 1_{z_{i}=z_{j}}}=0,
\end{equation*}
since $z_{1}\neq z_{2}$ a.s. under  $\bbP_{\not 12}$. Equation~\eqref{eq:mean:not12:blow-up} follows.

\paragraph{Variance under $\bbP_{\not 12}$.}
We follow the same approach as for  $\bbP_{12}$. 
Let $\pi^{(1)},\pi^{(2)}\in \Pi_{12}$ such that $|\text{range}(\pi^{(1)})\cap \text{range}(\pi^{(2)})|=2+u$, with $u\in \ac{0,\ldots,\kappa \gamma}$. Following~\eqref{eq:proof:ingredient:basis}, we again have 
$$\bbE_{\not 12}\cro{P_{G,\pi^{(1)}}P_{G,\pi^{(2)}}}=\lambda^{|E_{\Delta}|} \bar q^{|E_{\cap}|} \bbE_{\not 12}\cro{\prod_{(i,j)\in E_{\Delta}}\mathbf 1_{z_{i}=z_{j}}\prod_{\substack{(i,j)\in E_{\cap}\\ z_{i}=z_{j}}}{\bar p\over \bar q}}.$$
\smallskip

If $u=0$, we have $E_{\cap}=\emptyset$ and since $z_1\neq z_2$, we deduce that 
$\bbE_{\not 12}\cro{P_{G,\pi^{(1)}}P_{G,\pi^{(2)}}}=0$.  Then, for $u\in [1, \kappa \gamma]$, we argue as for $\bbE_{12}$. In particular, the counterpart of~\eqref{eq:up_E_12:blow-up} still holds up to the difference that under $\bbE_{\not 12}$, the number of communities that are distinct of $z_1$ or of $z_2$ is $\ell(z)-2$. Hence, we have 
\begin{align*}
\bbE_{\not 12}\cro{P_{G,\pi^{(1)}}P_{G,\pi^{(2)}}}
&\le \lambda^{|E_{\Delta}|}\bar p^{|E_{\cap}|}\frac{1}{K^{\kappa \gamma-u}}
\;\bbE_{\not 12}\cro{\left({\bar q \over \bar p}\right)^{(\ell(z)-2)_+(\gamma + a)} \prod_{(i,j)\in E^{(1)}_{\Delta}}\mathbf 1_{z_{i}=z_{j}}} \enspace . 
\end{align*}
Write again $\mathrm{cc}$ as the number of connected components of the graph $\overline{\mathcal{G}}^{(1)}_{\Delta}=(\pi^{(1)}(V),E^{(1)}_{\Delta})$ when we remove those containing $\pi^{(1)}(v_1)=1$ or  $\pi^{(1)}(v_2)=2$. Then, conditionally to $\prod_{(i,j)\in E^{(1)}_{\Delta}}\mathbf 1_{z_{i}=z_{j}}=1$, we have $(\ell(z)-2)\in [0,\mathrm{cc}]$ and for any $x\in [0,\mathrm{cc}]$ 
\[
\mathbb{P}_{\not 12}\cro{\ell(z)-2= x\,\Bigg|\, \prod_{(i,j)\in E^{(1)}_{\Delta}}\mathbf 1_{z_{i}=z_{j}}=1}\leq \mathrm{cc}^{\mathrm{cc}} K^{-(\mathrm{cc}-x)} \ . 
\]
Hence, we arrive as previously at 
\begin{align*}
\bbE_{\not 12}\cro{P_{G,\pi^{(1)}}P_{G,\pi^{(2)}}}&
 \leq \lambda^{2|E|}\left(\frac{\bar p}{\lambda^2}\right)^{|E_{\cap}|}\frac{1}{K^{2\kappa \gamma-u}} 
  u^u \sum_{x = 0 }^{u} \left(K \left(\frac{\bar q }{\bar p}\right)^{\gamma +a}\right)^{x }\ , 
\end{align*}
and the counterpart of~\eqref{eq:upper:covariance:u:blow-up} holds: 
\begin{align*}\nonumber
\bbE_{\not 12}\cro{P_{G,\pi^{(1)}}P_{G,\pi^{(2)}}}
&\leq \bbE^2_{12}\cro{P_{G,\pi}(Y)}  u^{u+1} \left[\left[K^2  \left(\frac{\bar{q}}{\lambda^2}\vee \frac{\bar{q}}{\bar p } \right)^{\gamma + a}\right]^u  + \left[K \left(\frac{\bar{p}}{\lambda^2}\vee 1 \right)^{\gamma + a}\right]^u   \right] \enspace . 
\end{align*}
Then, summing over all $\pi^{(1)}$ and $\pi^{(2)}$, we conclude as previously 
the bound on $\mathrm{var}_{\not 12}(R_{12})$
 in Equation~\eqref{eq:var:12:blow-up}.

\printbibliography

\end{document}